\documentclass[12pt]{article}
\PassOptionsToPackage{pdfpagelabels=false}{hyperref} 
\usepackage[english]{babel}
\usepackage{amsmath, amsxtra, amsfonts, amssymb, amstext}
\usepackage{amsthm}
\usepackage{booktabs}
\usepackage{fullpage}
\usepackage{nicefrac}
\usepackage{xspace}
\usepackage[noadjust]{cite}
\usepackage{url}
\usepackage{graphics,color}
\usepackage[colorlinks]{hyperref}
\definecolor{linkblue}{rgb}{0.1,0.1,0.8}
\hypersetup{colorlinks=true,linkcolor=linkblue,filecolor=linkblue,urlcolor=linkblue,citecolor=linkblue}
\usepackage[algo2e,ruled,vlined,linesnumbered]{algorithm2e}
\usepackage{multirow}
\newcommand{\assign}{\leftarrow}
\usepackage{pdfpages}

\usepackage{array}
\usepackage{bigdelim}

\newtheorem{theorem}{Theorem}
\newtheorem{lemma}[theorem]{Lemma}

\newtheorem{corollary}[theorem]{Corollary}

\newtheorem{remark}[theorem]{Remark}

\newcommand{\N}{\mathbb{N}}
\newcommand{\R}{\mathbb{R}}

\newcommand{\A}{\mathcal{A}}
\newcommand{\C}{\mathcal{C}}
\newcommand{\F}{\mathcal{F}}

\newcommand{\calE}{\mathcal{E}}

\renewcommand{\epsilon}{\varepsilon}
\newcommand{\eps}{\varepsilon}

\DeclareMathOperator{\E}{E}

\newcommand{\oea}{$(1 + 1)$~EA\xspace}

\newcommand{\onemax}{\textsc{OneMax}\xspace}
\newcommand{\OM}{\textsc{Om}\xspace}

\begin{document}

\title{OneMax in Black-Box Models with Several Restrictions}
\author{
Carola Doerr$^{1}$, 
Johannes Lengler$^2$}
\date{
$^1$CNRS and Sorbonne Universit\'es, UPMC Univ Paris 06, CNRS, LIP6 UMR 7606, 4 place Jussieu, 75005 Paris, France\\
$^2$Institute for Theoretical Computer Science, 
	ETH Z{\"u}rich, 
  Z{\"u}rich, Switzerland 
\\[2ex]
\today
}
\maketitle

\begin{abstract}
Black-box complexity studies lower bounds for the efficiency of general-purpose black-box optimization algorithms such as evolutionary algorithms and other search heuristics. Different models exist, each one being designed to analyze a different aspect of typical heuristics such as the memory size or the variation operators in use.
While most of the previous works focus on one particular such aspect, we consider in this work how the combination of several algorithmic restrictions influence the black-box complexity. Our testbed are so-called \textsc{OneMax} functions, a classical set of test functions that is intimately related to classic coin-weighing problems and to the board game Mastermind.

We analyze in particular the combined memory-restricted ranking-based black-box complexity of \textsc{OneMax} for different memory sizes. While its isolated memory-restricted as well as its ranking-based black-box complexity for bit strings of length $n$ is only of order $n/\log n$, the combined model does not allow for algorithms being faster than linear in $n$, as can be seen by standard information-theoretic considerations. We show that this linear bound is indeed asymptotically tight. Similar results are obtained for other memory- and offspring-sizes. Our results also apply to the (Monte Carlo) complexity of \textsc{OneMax} in the recently introduced elitist model, in which only the best-so-far solution can be kept in the memory. Finally, we also provide improved lower bounds for the complexity of \textsc{OneMax} in the regarded models.

Our result enlivens the quest for natural evolutionary algorithms optimizing \textsc{OneMax} in $o(n \log n)$ iterations.
\end{abstract}

\sloppy{
\section{Introduction}
\label{sec:Intro}

Black-box complexity aims at analyzing the influence of algorithmic choices such as the population size, the variation operators in use, or the selection principles on the optimization time of evolutionary algorithms (EAs) and other (deterministic or randomized) search heuristics. Lower bounds from black-box complexity theory provide information about the limits of certain classes of evolutionary algorithms (e.g., memory-restricted, ranking-based, or unbiased EAs), while upper bounds can serve as an inspiration for the development of new EAs.

Unlike other complexity notions, black-box complexity is a measure for the number of \emph{black-box queries} that an algorithm does in order to optimize an unknown function $f$. That is, we simply count the number of function evaluations $f(x)$ that are needed (usually in expectation, but cf.~Section~\ref{sec:lasvegas} below) until for the first time an optimal search point $x \in \arg\max f$ is evaluated. Black-box complexity typically disregards all computational efforts that an algorithm executes between any two different such function evaluations. In ``classical'' theoretical computer science (TCS) black-box complexity is often referred to as \emph{(randomized) query complexity.} While being a well-known complexity notion in there, the focus in the broader TCS community is typically on having a simplified complexity measure for sorting, coin-weighing, and other problems, and not, as is the case in evolutionary computation, on analyzing the impact of above-mentioned algorithmic choices on the performance of general-purpose problem solvers. 

\subsection{Related Work on Black-Box Complexity}

In the context of evolutionary computation (EC), black-box complexity has first been studied by Droste, Jansen, (Tinnefeld,) and Wegener in~\cite{DrosteJTW03} and~\cite{DrosteJW06}. The authors regard two different black-box models, an \emph{unrestricted} version, in which the algorithms have arbitrary memory and full access to function values, and a \emph{memory-restricted} one, in which the algorithms are allowed to store only a limited number of previously queried search points and their function values. While the unrestricted model is mostly used for analyzing lower bounds, the memory-restricted model is studied in the context of upper bounds. Since most EAs have a limited population size, these upper bounds typically provide a better comparison for the efficiency of different algorithmic approaches. 

The theory seemed to have come to an early end afterward since even the memory-restricted version yielded black-box complexities that were unreasonably low compared to the performance of evolutionary algorithms. Thus, the notion appeared to be of little use for the understanding of such algorithms. However, the field experienced a major revival with the work of Lehre and Witt~\cite{LehreW10, LehreW12} on the \emph{unbiased black-box model.} In this version, new search points can be obtained by the algorithm only by sampling uniformly at random from the underlying search space, or (for the search space being the $n$-dimensional hypercube $\{0,1\}^n$) by combining previously queried search points in a way that does not discriminate between the bit positions $1,2,\ldots, n$ nor between the bit values $0$ and $1$. Many EAs use \emph{variation operators} of this unbiased type. 

Lehre and Witt could show that their unbiased black-box complexity notion can give much better estimates for the efficiency of typical EAs than the previous models. This also applies to the so-called \onemax problem, whose unrestricted black-box complexity is only of order $n/\! \log n$~\cite{ErdR63, DrosteJW06, AnilW09} while its unary unbiased black-box complexity is of order $n \log n$~\cite[Theorem 6]{LehreW12}, thus matching the expected optimization time of search heuristics such as the so-called \oea and Randomized Local Search. The (generalized) \onemax problem is to identify an unknown bit string $z$ if with each each query $x$ the algorithm learns the number $\OM_z(x):=|\{ i \in \{1,2,\ldots, n\} \mid x_i=z_i\}|$ of bit positions in which $x$ and $z$ agree (in other words, $\OM_z(x)$ equals $n$ minus the Hamming distance of $x$ and $z$). This problem can be seen as a generalization of the popular Mastermind game with two colors (cf.~\cite{DoerrW14memory}), and is one of the easiest pseudo-Boolean optimization problems as it only requires trap-free hill-climbing. As such it is typically one of the first test problems that is regarded when introducing a new black-box model.

It was left as an open question in~\cite{DrosteJW06} whether or not restricting the memory of an algorithm already yields a similar runtime bound of $\Omega(n \log n)$ for the optimization of \onemax. This hope was dashed in~\cite{DoerrW14memory}, where it has been shown that even for the smallest possible memory size, in which algorithms may store only one previously queried search point and its fitness, an $O(n/\! \log n)$ algorithm exists. 
Similarly, in the \emph{ranking-based black-box model}, in which the algorithms learn only the ranking of the function values, but not their absolute values, \onemax can still be solved in an expected number of $O(n/\! \log n)$ function evaluations~\cite{DoerrW14ranking}. 

\subsection{Our Results}

While previous work in black-box complexity theory focused on analyzing the influence of single restrictions on the efficiency of the algorithms under consideration, we regard in this work \emph{combinations} of such algorithmic constraints. As testbed, we regard the above-mentioned class of \onemax functions. Since for this problem many, often provably tight, bounds are available for the single-restriction models, we can easily compare our results to see how the combined restrictions impact the best-possible optimization times, cf. Table~\ref{tab:onemax} below. 

In a first step, we study the combined memory-restricted ranking-based model, i.e., we study the black-box complexity of \onemax with respect to $(\mu+\lambda)$ memory-restricted ranking-based algorithms. Algorithms fitting this framework are allowed to store up to $\mu$ previously queried search points and their ranking with respect to the underlying objective function $f$. (Solely) from this information, the algorithms then generate and query $\lambda$ new search points (so-called \emph{offsprings}). They receive information of how these newly generated search points perform with respect to the parent population (more precisely, the full ranking of the $\mu+\lambda$ search points with respect to $f$ is revealed to the algorithms), and the algorithms then select an arbitrary subset of $\mu$ of these search points, which form the parent population of the next iteration. This process continues until a search point $x \in \arg\max f$ is queried for the first time.    

For the most restrictive case $\mu=\lambda=1$ (i.e., the often regarded (1+1) scheme), the algorithms under consideration are easily seen to be \emph{comparison-based}, i.e., they learn with each query only whether the offspring has better, equal, or worse fitness than its parent. Therefore, by a simple information-theoretic argument (cf., e.g., \cite[Theorem 2]{DrosteJW06}), their expected optimization time on \onemax is at least linear in $n$. This already shows that the combined (1+1) memory-restricted ranking-based black-box complexity of \onemax is asymptotically larger than either the pure ranking-based or the pure memory-restricted version. However, this is not the end of the story. In this work we show lower bounds for the combined (1+1) model that are by a constant factor stronger than the best known bounds for comparison-based algorithms. Thus they are stronger than any bound obtained by reducing the combined model to an existing black-box model with a single restriction. On the other hand, we show that the mentioned linear lower bound is asymptotically tight. That is, we provide a linear time algorithm solving \onemax in a (1+1) scheme and using only relative fitness information. Also for many other combinations of $\mu$ and $\lambda$ we show that the information-theoretic lower bound is matched by a $(\mu+\lambda)$ memory-restricted ranking-based algorithm. 

In a next step, we combine the memory-restricted ranking-based model with yet another restriction, namely with the recently introduced elitist selection requirement introduced in the work~\cite{DoerrL15Model} on \emph{elitist black-box models}. In this context, we additionally require that the algorithm selects the $\mu$ \emph{fittest} individuals out of the $\mu+\lambda$ parents and offspring (where it may break ties arbitrarily).\footnote{As mentioned in~\cite{DoerrL15Model} we remark that the usage of ``elitist selection'' is not standardized in the EA literature. Some subcommunities would therefore rather call our elitist black-box model a \emph{black-box model with truncation selection}.} Notably, the achievable optimization times stays the same (asymptotically), though in a slightly different sense as we shall discuss below. This is rather surprising, as all previous black-box optimal algorithms make substantial use of non-elitist selection. 

Table~\ref{tab:onemax}, taken from~\cite{DoerrD14tutorial} and extended to cover the results of the present paper, summarizes known lower and upper bounds of the complexity of \onemax in the different black-box models. Bounds given without reference follow trivially from identical bounds in stronger models, e.g., the $\Omega(n / \log n)$ lower bound for the memory-restricted black-box complexity follows directly from the same bound for the unrestricted model. 

A short version of this work has been presented at the Genetic and Evolutionary Computation Conference (GECCO 2015) in Madrid, Spain~\cite{DoerrL15OM}.

\begin{table*}
\begin{center}
\begin{scriptsize}
\begin{tabular}{l|ll|ll}
\textbf{Model}    
& \multicolumn{2}{c|}{\textbf{Lower Bound}} 
& \multicolumn{2}{c}{\textbf{Upper Bound}}\\ \hline 
unrestricted 
	& $\Omega(n/\! \log n)$ & info.-theo.
	& $O(n/\! \log n)$ &\cite{ErdR63, AnilW09} \\
	\hline
unbiased, arity $1$
	& $\Omega(n \log n)$  &\cite{LehreW12}
	& $O(n \log n)$  &
	\\ 
unbiased, arity $2 \leq k \leq \log n$ 
	& $\Omega(n/\! \log n)$ 
	&
	& $O(n/k)$  &\cite{DoerrW14arity, DoerrDE15}
	\\ \hline
r.b. unrestricted  
	& $\Omega(n/\! \log n)$  
	&
	& $O(n/\! \log n)$  &\cite{DoerrW14ranking}
	\\ 	
r.b. unbiased, arity  $1$
	& $\Omega(n \log n)$  
	&
	& $O(n \log n)$  
	& 
	\\ 
r.b. unbiased, arity $ 2 \leq k \leq n$
	& $\Omega(n/\! \log n)$  
	&
	& $O(n /\log k)$  
	&\cite{DoerrW14ranking}
	\\
	\hline
(1+1) comparison-based
	& $\Omega(n)$  
	&info.-theo.
	& $O(n)$  
	&
	\\ \hline				
(1+1) memory-restricted  
	& $\Omega(n/\! \log n)$  
	&
	& $O(n /\log n)$  
	&\cite{DoerrW14memory}
	\\ \hline				
(1+1) elitist Las Vegas
	& $\Omega(n)$  
	& \rdelim\}{5}{1.8cm}[\normalfont Thm.~\ref{thm:lower}]
	& $O(n\log n)$  
	&
	\\ 				
(1+1) elitist $\log n/n$-Monte Carlo
	& $\Omega(n)$  
	&
	& $O(n)$  
	&Thm.~\ref{thm:11upper}
	\\ 				
(2+1) elitist Monte Carlo/Las Vegas
	& $\Omega(n)$  
	&
	& $O(n)$  
	&Thm.~\ref{thm:21upper}
	\\ 				
(1+$\lambda$) elitist Monte Carlo ($\#$ generations)	
	& $\Omega(n/\! \log \lambda)$  
	&
	& $O(n/\! \log \lambda)$ 
	&Thm.~\ref{thm:1lambdaupper}
	\\ 		
($\mu$+1) elitist Monte Carlo
	& $\Omega(n/\! \log \mu)$  
	&
	& $O(n/\! \log \mu)$   
	&Thm.~\ref{thm:mu1upper}
	\\ 		
($1,\lambda$) elitist Monte Carlo/Las Vegas ($\#$ generations)
	& $\Omega(n/\! \log \lambda)$  
	& cf.\! Section~\ref{sec:comma}
	& $O(n/\! \log \lambda)$  
	&Thm.~\ref{thm:11comma}
\end{tabular}
\end{scriptsize}
\end{center}
\caption{The black-box complexities of \onemax in the different models. r.b. abbreviates ranking-based; 
info-theo. the information-theoretic bound~\cite{Yao77}, cf. also~\cite{DrosteJW06}; for $(1+\lambda)$ and $(1,\lambda)$ we assume $1< \lambda < 2^{n^{1-\eps}}$ for some $\eps>0$, and for $(\mu+1)$ we assume that $\mu = \omega(\log^2 n/\! \log \log n)$ and $\mu \leq n$.}
\label{tab:onemax}
\end{table*}

\subsection{Relevance of Our Work and Techniques}

While at a first glance the obtained upper bounds may seem to be a shortcoming of the model (most EAs need $\Omega(n \log n)$ steps to optimize \onemax-functions), it does not have to be. In light of~\cite{DoerrDE15}, where a simple and natural EA has been designed that optimizes \onemax in $o(n \log n)$ time, it is well possible that such a result can be extended further (of particular interest is an extension to (1+1)-type algorithms). As we know from~\cite{DoerrDE15}, black-box complexity results like our mentioned \onemax bound can give an inspiration for developing such algorithms.  

One obvious challenge for designing algorithms in the combined memory-restricted ranking-based model is the fact that the best-known algorithms in the single-restriction case either make heavy use of knowing the absolute fitness values (in the memory-restricted case, see~\cite{DoerrW14memory}) or of having access to a large number of previously queried search points (in the ranking-based case, cf.~\cite{DoerrW14ranking}). It is thus not obvious how to design efficient algorithms respecting both restrictions at the same time. Our results therefore require approaches and strategies that are significantly different from those found in previous works, though, at the other hand, we can and also do make significant use of several ideas developed in previous works on \onemax in the different black-box models. For example, for the (1+1) memory-restricted ranking-based elitist black-box model the algorithm certifying the linear upper bound nicely combines previous techniques from the black-box complexity literature with some newly developed tools such as the \emph{neutral counters} designed in Section~\ref{sec:counter}.
We believe that the insights from these tools will be useful in future research in evolutionary computation, both in algorithm analysis and in algorithm design.

For the \emph{lower bounds,} a technical difficulty that we face in the proofs is a putative non-applicability of Yao's Principle. More precisely, there may be randomized algorithms that even in the worst case perform much better than any deterministic algorithm on a random problem instance, cf.\! Section~\ref{sec:Yao} and~\cite{DoerrL15Model}. We overcome these problems by expanding the class of algorithms regarded. This needs some care as we do not want to decrease the complexity too much by this expansion. 

\subsection{Structure of the Paper}
Our paper is structured as follows. We start with a formal introduction of the models in Section~\ref{sec:model}, followed by a brief discussion on the difference between Las Vegas and Monte Carlo complexities, which can be crucially different in memory-restricted models. In a nutshell, the \emph{Las Vegas} complexity measures the expected time until an optimal search point is hit, while the \emph{$p$-Monte Carlo complexity} asks for the time needed until an optimum is hit with probability at least $1-p$. These bounds can be exponentially far apart as shown in~\cite{DoerrL15Model} and thus need to be regarded separately. In Section~\ref{sec:onemaxbackground} we formally introduce the generalized \onemax functions and recapitulate the known bounds on its complexity in different black box models. We conclude the introductory sections by providing some basic tools in Section~\ref{sec:tools}.

In Section~\ref{sec:lower} we provide the mentioned lower bounds for the $(\mu+\lambda)$ memory-restricted ranking-based black-box complexity of \onemax for a wide range of $\mu$ and $\lambda$. For the upper bounds, most of our proofs work directly in the elitist black box model, so the remainder of the paper is devoted to the proofs of such upper bounds in the elitist model, which imply the same upper bounds for the memory-restricted ranking-based model. We first give a simple linear upper bound for the (2+1) (Las Vegas and Monte Carlo) elitist black-box complexity of \onemax (Section~\ref{sec:21upper}). At the heart of this paper is Section~\ref{sec:11upper}, where we show the linear upper bound for the (1+1) Monte Carlo elitist black-box complexity of \onemax. In Sections~\ref{sec:1lambdaupper} and~\ref{sec:mu1upper}, we consider more generally $(1+\lambda)$ and $(\mu+1)$ elitist black-box algorithms. Finally, in Section~\ref{sec:comma} we give some remarks on the $(\mu,\lambda)$ elitist black-box complexities of \onemax and point out some important differences from the $(\mu+\lambda)$ complexities.

\section{Black-Box Models and Complexity Measures}
\label{sec:model}

We are primarily interested in analyzing the memory-restricted ranking-based black-box complexities of \onemax. An important difference to purely memory-restricted algorithms is that the available memory is strictly smaller in this combined memory-restricted and ranking-based model. If we regard, for example, the (1+1) case, then in the purely memory-restricted model the algorithm does not only have access to the current search point, but also to its fitness value. It thus has strictly more than $n$ bits of information when sampling the offspring. If, on the other hand, the algorithm is in addition also ranking-based, then it may not access the fitness; thus its available information is restricted to exactly $n$ bits. So the fitness-based variant has effectively a larger available memory than the ranking-based one (but of course both are not completely free in how to use the memory).

Formally, a $(\mu+\lambda)$ memory-restricted, ranking-based black-box algorithm maintains a population (parent generation) of $\mu$ search points, and knows the ranking of their fitnesses. Based solely on this information it samples $\lambda$ additional search points (offsprings), and receives the ranking of all $\mu+\lambda$ fitnesses. From the parent generation and the offsprings, it selects $\mu$ search points to form the new parent generation. A $(\mu+\lambda)$ memory-restricted, ranking-based black-box algorithm is~\emph{elitist} if in the selection step it selects the $\mu$ \emph{best} search points with respect to the ranking. The algorithm may break ties arbitrarily: for example, if all $\mu+\lambda$ search points have the same fitness, then it may choose an arbitrary subset of size $\mu$ to form the next parent generation. The formal structure of a $(\mu+\lambda)$ elitist  black-box algorithm is given by Algorithm~\ref{alg:elitist}.


\begin{algorithm2e}
 \textbf{Initialization:} \\
 \Indp
 $X \assign \emptyset$\;
 \For{$i=1,\ldots,\mu$}{
 Depending only on the multiset $X$ and the ranking $\rho(X,f)$ of $X$ induced by $f$, choose a probability distribution $p^{(i)}$ over $\{0,1\}^n$ and sample $x^{(i)}$ according to $p^{(i)}$\;
 $X \assign X \cup \{ x^{(i)}\}$\;
 }
 \Indm
 \textbf{Optimization:}	
 \For{$t=1,2,3,\ldots$}{
 		Depending only on the multiset $X$ and the ranking $\rho(X,f)$ of $X$ induced by $f$
	\label{line:mut}	choose a probability distribution $p^{(t)}$ on $(\{0,1\}^n)_{i=1}^{\lambda}$ and 
		sample $(y^{(1)},\ldots,y^{(\lambda)}) $ according to $p^{(t)}$\;
		Set $X \assign X \cup \{y^{(1)},\ldots,y^{(\lambda)}\}$\;
  \lFor{$i=1,\ldots, \lambda$}{
  	\label{line:selection} Select $x \in \arg\min X$ and update $X \assign X \setminus \{x\}$}}
 \caption{The $(\mu+\lambda)$ elitist black-box algorithm for maximizing an unknown function $f:\{0,1\}^n \rightarrow \R$}
\label{alg:elitist}
\end{algorithm2e}
 
 Note that the only difference to the $(\mu+\lambda)$ memory-restricted ranking-based black-box model is the enforced elitist selection in line~\ref{line:selection}, which in the former model can be replaced by 
\begin{align*}
  \text{\textbf{for} } i=1,\ldots, \lambda \text{\textbf{ do } }
  \text{Select $x \in X$ and update $X \assign X \setminus \{x\}$;}
\end{align*}
Since the elitist model is more restrictive than the combined memory-restricted ranking-based one, every upper bound on the $(\mu+\lambda)$ elitist black-box complexity also holds for the  $(\mu+\lambda)$ memory-restricted ranking-based black-box complexity. As discussed in~\cite{DoerrL15Model} several variants of the elitist model exist, but this is beyond the scope of the present paper. 

Following the standard convention for black-box optimization, we define the \emph{runtime} (or \emph{optimization time}) of a $(\mu+\lambda)$ black-box algorithm $A$ to be the number of search points sampled by $A$ until an optimal search point is sampled for the first time (samples are counted with multiplicities if they are sampled several times). Since a $(\mu+\lambda)$ algorithm samples $\lambda$ search points in each generation, the runtime of an algorithm after $t$ generations is $\mu + \lambda t$, but see also our comment at the end of Section~\ref{sec:lasvegas}. 

\subsection{Las Vegas vs. Monte Carlo Complexities}
\label{sec:lasvegas}

Elitist black-box algorithms cannot do simple restarts since a solution intended for a restart is not allowed to be accepted into the population if its fitness is not as good as those of the search points currently in the memory. Regarding expected runtimes can therefore be significantly different from regarding algorithms with allowed positive failure probability. In fact, it is not difficult to see that these two notions can be exponentially far apart~\cite[Theorem 3]{DoerrL15Model}. One may argue that this is a rather artificial problem since in practice there is no reason why one would not want to allow restarts. Also, almost all algorithms used to show upper bounds in the previous black-box models have small complexity only because of the possibility of doing random restarts. One convenient way around this problem is to allow for small probabilities of failure. Such (high) probability statements are actually often found in the evolutionary computation literature. The following definition captures its spirit.

Let us regard for a black-box algorithm $A$ the smallest number $T$ of function evaluations that is needed such that for any problem instance $f \in \F$ the optimum of $f$ is found with probability at least $1-p$. We call $T=T(A,\F)$ the $p$-Monte Carlo black-box complexity of $A$ on $\F$. The \emph{$p$-Monte Carlo black-box complexity} of $\F$ with respect to a class $\A$ of algorithms is $\min_{A \in \A}T(A,\F)$. If we make a statement about the Monte Carlo complexity without specifying $p$, then we mean that for \emph{every constant} $p>0$ the statement holds for the $p$-Monte Carlo complexity. However, we sometimes also regard $p$-Monte Carlo complexities for non-constant $p=p(n) = o(1)$, thus yielding high probability statements. 

The standard black-box complexity (which regards the maximal \emph{expected} time that an algorithm $A$ needs to optimize any $f \in \F$) is called \emph{Las Vegas black-box complexity} in~\cite{DoerrL15Model}. We adopt this notation. 

We recall from~\cite{DoerrL15Model} that, by Markov's inequality, every Las Vegas algorithm is also (up to a factor of $1/p$ in the runtime) a $p$-Monte Carlo algorithm. We also repeat the following statement which is a convenient tool to bound $p$-Monte Carlo complexities.

\begin{remark}[Remark 1 in~\cite{DoerrL15Model}]
\label{rem:1}
Let $p \in (0,1)$. Assume that there is an event $\calE$ of probability $p_{\calE} < p$ such that conditioned on $\neg \calE$ the algorithm $A$ finds the optimum after \emph{expected} time at most $T$. Then the $p$-Monte Carlo complexity of $A$ on $f$ is at most $(p-p_\calE)^{-1}T$. In particular, if $p-p_\calE = \Omega(1)$ then the $p$-Monte Carlo complexity is $O(T)$.
\end{remark}

For some applications it is more natural to count the number of generations rather than the number of sampled search points (e.g., because the evaluations of different search points may be parallelizable). For this reason, we give some complexities also for the number of generations, cf.~Table~\ref{tab:onemax}. All definitions above transfer analogously, with the runtime of an algorithm replaced by the number of generations needed before an optimal search point is sampled for the first time. However, note that all black-box complexities refer to the expected runtime unless explicitly stated otherwise.

\section{Background on OneMax Complexities and Overview of Results}
\label{sec:onemaxbackground}

One of the most prominent problems in the theory of randomized search heuristics is the running time of evolutionary algorithms and other heuristics on the \onemax problem. \onemax is the function that counts the number of ones in a bitstring. Maximizing \onemax thus corresponds to finding the all-ones string. 

Search heuristics are typically invariant with respect to the problem encoding, and as such they have the same runtime for any function from the generalized \onemax function class 
\begin{align*}
\onemax:=\left\{\OM_z \mid z \in \{0,1\}^n \right\},
\end{align*}
where $\OM_z$ is the function
\begin{align}
\label{def:OMz}
\OM_z:\{0,1\}^n \rightarrow \R, x \mapsto n-\sum_{i=1}^n{(x_i \oplus z_i)},
\end{align} 
assigning to $x$ the number of positions in which $x$ and $z$ agree. We call $z$, the unique global optimum of function $\OM_z$, the \emph{target string} of $\OM_z$. 
Whenever we speak of the \onemax problem or a \onemax function we mean the whole class of \onemax functions or an unknown member of it, respectively. 
 
The \onemax problem is by far the most intensively studied problem in the runtime analysis literature and, due to its close relation to the classic board game \emph{Mastermind}~\cite{DoerrW14memory}, to cryptographic applications, and to coin-weighing problems, it is also studied in other areas of theoretical computer science. 
Also for black-box complexities it is the most commonly found test problem. Without going too much into detail, we recall that the unrestricted black-box complexity of \onemax is $\Theta(n/\! \log n)$~\cite{DrosteJW06,AnilW09,ErdR63}. While the lower bound is a simple application of Yao's Principle (Lemma~\ref{lem:Yao}, cf.~\cite{DrosteJW06} for a detailed explanation of the $\Omega(n/\! \log n)$ lower bound), the upper bound is achieved by an extremely simple, yet elegant algorithm: sampling $O(n/\! \log n)$ random search points and regarding their fitness values, with high probability, reveals the target string $z$. We shall make use of (variants of) this strategy in some of our proofs of upper bounds.
 
Another important bound for the \onemax problem is the simple $\Theta(n)$ bound for comparison-based algorithms as introduced in~\cite{TeytaudG06}.\footnote{The lower bound is again a simple application of Yao's Principle (Lemma~\ref{lem:Yao}), while the upper bound is attained, for example, by the algorithm which checks one bit at a time, going through the bitstring from one end to the other. Alternatively, the upper bound is also verified by the $(1+(\lambda,\lambda))$~GA of~\cite{DoerrDE15}, thus showing that it can also be achieved by unbiased algorithms of arity two.} Since (1+1) memory-restricted ranking-based algorithms are comparison-based, this gives a linear lower bound for their complexity on \onemax. 
\begin{remark}
\label{rem:OMlowercb}
The (1+1) memory-restricted ranking-based black-box complexity of \onemax is $\Omega(n)$, thus implying a linear lower bound for the (1+1) elitist Las Vegas and Monte Carlo black-box complexity of \onemax.
\end{remark}
If we consider the leading constants hidden in the $\Omega$-notation, then the lower bounds coming from the comparison-based complexity are not optimal. In Theorem~\ref{thm:lower} we will prove lower bounds for memory-restricted ranking-based algorithms that are by a non-trivial constant factor higher than the best known bounds for comparison-based algorithms.

Our upper bounds will show that there are elitist black-box optimization algorithms optimizing \onemax much more efficiently than typical heuristics like RLS or evolutionary algorithms. In particular we show that the (1+1) elitist Monte Carlo black-box complexity is at most linear (which is best possible by Theorem~\ref{thm:lower}). Our results  are summarized in the lower part of Table~\ref{tab:onemax}. Note that the upper bounds for elitist algorithms immediately imply upper bounds for the (Monte Carlo and Las Vegas) black-box complexity of \onemax in the respective memory-restricted ranking-based models. The lower bounds also carry over in asymptotic terms (i.e., up to constant factors), cf.\! Theorem~\ref{thm:lower}. Since the memory-restricted ranking-based bounds were the original motivation for our work, we collect them in the following statement. 

\begin{corollary}
\label{cor:mrrb}
The (1+1) memory-restricted ranking-based (Las Vegas) black-box complexity of  \onemax is $\Theta(n)$. 
For $1<\lambda < 2^{n^{1-\eps}}$, $\eps>0$ being an arbitrary constant, its $(1+\lambda)$ memory-restricted ranking-based black-box complexity is $\Theta(n/\! \log \lambda)$ (in terms of generations), 
while for $\mu=\omega(\log^2(n)/\log\log n)$ its $(\mu+1)$ memory-restricted ranking-based black-box complexity is $\Theta(n/\! \log \mu)$.\footnote{We do not consider in this work $(\mu+\lambda)$ elitist algorithms for $\mu$ and $\lambda$ both being strictly greater than one. We feel that the required tools are given in the $(1+\lambda)$ and $(\mu+1)$ settings, so that analyzing the additional settings would not give sufficiently many new insights.}
\end{corollary}

\section{Tools}
\label{sec:tools}

In this section we list some tools that we need to study the $(\mu+\lambda)$ memory-restricted ranking-based black-box and the $(\mu+\lambda)$ elitist black-box complexity. More precisely, we recapitulate the RLS algorithm, Yao's principle, and a Negative Drift Theorem. 

\subsection{Random Local Search}

A very simple heuristic optimizing \onemax in $\Theta(n \log n)$ steps is \emph{Randomized Local Search} (RLS). Since this heuristic will be important in later parts of this paper, we state it here for the sake of completeness. 
RLS, whose pseudo-code is given in Algorithm~\ref{alg:RLS}, is initialized with a uniform sample $x$. In each iteration one bit position $j \in [n]:=\{1,2,\ldots,n\}$ is chosen uniformly at random. The $j$-th bit of $x$ is flipped and the fitness of the resulting search point $y$ is evaluated. The better of the two search points $x$ and $y$ is kept for future iterations (favoring the newly created individual in case of ties). As is easily verified, RLS is a unary unbiased (1+1) elitist black-box algorithm, where we understand \emph{unbiasedness} in the sense of Lehre and Witt~\cite{LehreW12}.

\begin{algorithm2e}%
	\textbf{Initialization:} Sample $x \in \{0,1\}^n$ uniformly at random and query $f(x)$\;
 \textbf{Optimization:}
\For{$t=1,2,3,\ldots$}{
Choose $j \in [n]$ uniformly at random\;
Set $y\assign x\oplus e^n_{j}$ and query $f(y)$\,; //mutation step\\
\lIf{$f(y)\geq f(x)$}{$x \assign y$\,; //selection step}
}
\caption{Randomized Local Search for maximizing~$f\colon\{0,1\}^n\to\mathbb{R}$.}
\label{alg:RLS}
\end{algorithm2e}

\subsection{Yao's Principle}
\label{sec:Yao}

We will use the following formulation of Yao's principle. See~\cite{DoerrL15Model} for a more detailed exposition of Yao's principle in the context of elitist black-box complexity.
\begin{lemma}[Yao's Principle~\cite{Yao77, MotwaniR95}]
\label{lem:Yao}
Let $\Pi$ be a problem with a finite set $\mathcal I$ of input instances (of a fixed size) permitting a finite set $\A$ of deterministic algorithms. Let $p$ be a probability distribution over $\mathcal I$ and $q$ be a probability distribution over $\A$. Then, 
\begin{align}\label{eq:Yao}
	\min_{A \in \A} \E[T(I_p, A)] \leq \max_{I \in \mathcal I} \E[T(I,A_q)]\, , 
\end{align}
where $I_p$ denotes a random input chosen from $\mathcal I$ according to $p$, $A_q$ a random algorithm chosen from $\C$ according to $q$ and $T(I,A)$ denotes the runtime of algorithm $A$ on input $I$.
\end{lemma}

For most problem classes Yao's principle implies that the runtime $T$ of a best-possible \emph{deterministic} algorithms on a \emph{random} input is a lower bound to the best-possible performance of a \emph{random} algorithm on an \emph{arbitrary} input. However, this is \emph{not} true for $(\mu+\lambda)$ memory-restricted or elitist algorithms, since there are randomized memory-restricted (or elitist) algorithms that are not convex combinations of deterministic ones (i.e., that can not be obtained by deciding randomly on one deterministic algorithm, and then running this algorithm on the input). 

For example, every deterministic (1+1) memory-restricted ranking-based algorithm that ever rejects a search point (i.e., does not go to the newly sampled search point) will be caught in an infinite loop on \onemax with positive probability if the input is chosen uniformly at random. Hence, such an algorithm will have infinite expected runtime. On the other hand, if the algorithm does not reject any search point, then it is easy to see that its expected runtime on \onemax is $\Omega(2^n)$. However, there are certainly (1+1) memory-restricted ranking-based randomized algorithms (e.g., RLS) that optimize \onemax in expected time $O(n\log n)$. We refer the reader to~\cite{DoerrL15Model} for a more detailed discussion. To solve this putative non-applicability of Yao's Principle (cf. again~\cite{DoerrL15Model} for a more detailed discussion), we apply it to a suitable superset of algorithms. In particular, Yao's principle applies to every set of algorithm that have access to their whole search histories. 

\subsection{Negative Drift}

We recall the Negative Drift Theorem as given in~\cite{OlivetoW11}.

\begin{theorem}[Negative Drift Theorem~\cite{OlivetoW11}]
\label{thm:negativedrift}
Let $X_t$, $t\geq 0$ be real-valued random variables describing a stochastic process over some state space, with filtration $\mathcal{F}_t := (X_0,\ldots,X_t)$. Suppose there exists an interval $[a,b]\subseteq \R$, two constants $\delta,\eps>0$ and, possibly depending on $\ell:= b-a$, a function $r(\ell)$ satisfying $1\leq r(\ell) = o(\ell/\log \ell )$ such that for all $t\geq 0$ the following two conditions hold:
\begin{enumerate}
\item $\E[X_{t}-X_{t+1} \mid \mathcal{F}_t \wedge a < X_t <b] \leq -\eps$,
\item $\Pr[|X_{t}-X_{t+1}| \geq j \mid \mathcal{F}_t \wedge a < X_t] \leq \frac{r(\ell)}{(1+\delta)^j}$ for $j\in\N_0$.
\end{enumerate}
Then there is a constant $c^* >0$ such that for $T^* := \min\{t\geq 0: X_t \leq a \mid \mathcal{F}_t \wedge X_0 \geq b\}$ it holds $\Pr[T^* \leq 2^{c^*\ell/r(\ell)}] = 2^{-\Omega(\ell/r(\ell))}$.
\end{theorem}

\section{Lower Bounds}
\label{sec:lower}

In this section we show that the (1+1) memory-restricted ranking-based black-box complexity of \onemax is at least $\Omega(n)$.  In fact, we show this bound for a large range of function classes. We also show (mostly tight, as the algorithms in subsequent sections will show) lower bounds for general $(\mu+\lambda)$ elitist black-box algorithms. 

We use Yao's Principle (Theorem~\ref{lem:Yao} in Section~\ref{sec:Yao}). However, as outlined in Section~\ref{sec:Yao}, Yao's Principle is not directly applicable to memory-restricted or elitist black-box algorithms. Still we can apply Yao's Principle to a suitable superset of algorithms, yielding the following bounds.

\begin{theorem}
\label{thm:lower}
Let $\F$ be a class of functions such that for every $z \in \{0,1\}^n$ there is a function $f_z\in \F$ with unique optimum $z$. Then the (1+1) memory-restricted ranking-based black-box complexity of $\F$ (and thus, also the elitist (1+1) Las Vegas black-box complexity) is at least $n-1$. Moreover, for every $p>0$ the $p$-Monte Carlo black-box complexity of $\F$ is at least $n+ \lceil \log (1-p)\rceil$.

In general, for every $\mu \geq 1$ and $\lambda \geq 1$, the following statements are true for the memory-restricted ranking-based black box complexity, for the elitist Las Vegas black box complexity, and for the elitist Monte Carlo black box complexity. 
\begin{itemize}
\item The $(1+\lambda)$ black-box complexity of $\F$ is at least $n/\! \log(\lambda+1) - O(1)$. 
\item The $(\mu + 1)$ black-box complexity of $\F$ is at least $n/\! \log(2\mu+1) - O(1)$. 
\item The $(\mu + \lambda)$ black-box complexity of $\F$ is at least $n/(b+o(1))$, where $b= \log(\binom{\mu +\lambda}{\mu}) + \mu(\log \mu -1 - \log\ln 2)-1$. 
\end{itemize}
\end{theorem} 

\begin{proof}[Proof of Theorem~\ref{thm:lower}]\hspace{-1ex}\footnote{The extended abstract~\cite{DoerrL15OM} published at GECCO contains a proof that covers only the elitist case, but is more intuitive and less technical. We advice the reader who is only interested in the proof ideas to read that proof rather than the general version given here.}
We first give the argument for the (1+1) case to elucidate the argument, although this case is covered by the more general $(1+\lambda)$ case. We use Yao's Principle on the set $\A'$ of all algorithms $A$ satisfying the following restrictions. $A$ is a comparison-based (1+1) black-box algorithm that has access to the whole search history. (Thus we may apply Yao's Principle, see Section~\ref{sec:Yao}.) The algorithm learns about $f$ by oracle queries of the following form. It may choose a search point $x$ that it has queried before (in the first round, it simply chooses a search point without querying), and a search point $y$. Then $A$ may choose a subset $S$ of $\{\text{``$<$''}, \text{``$=$''}, \text{``$>$''}\}$ and the oracle will return \emph{yes} if the relation between $f(x)$ and $f(y)$ is in $S$, and \emph{no} otherwise. For example, if $S = \{\text{``$<$''}, \text{``$=$''}\}$ then the oracle answers the question ``Is $f(x) \leq f(y)$?''.

Let $\A$ be the set of all (1+1) memory-restricted ranking-based black-box algorithms. We need to show $\A \subseteq \A'$, so let $A \in \A$. When the current search point of $A$ is $x$, the algorithm may first decide on the next search point $y$ (i.e., it assigns to each search point $y$ a probability $p_y$ to be queried). If the oracle (of model $\A$) tells the algorithm ``$f(x) < f(y)$'', then $A$ may choose to stay in $x$ with some probability $p_\ell$ and to go to $y$ with probability $1-p_\ell$. Similarly, let $p_e$ and $p_g$ be the probability that the algorithm stays in $x$ if the oracle responds ``$f(x) = f(y)$'' or ``$f(x) > f(y)$'', respectively.

We may simulate $A$ in the model $\A'$ as follows. We first choose the point $y$ with probability $p_y$ as $A$ does. Then we set $S$ to be $\{\text{``$<$''}, \text{``$=$''}, \text{``$>$''}\}$ with probability $p_\ell \cdot p_e\cdot p_g$, the set $\{\text{``$<$''}, \text{``$=$''}\}$ with probability $p_\ell \cdot p_e\cdot (1-p_g)$, and so on. (I.e., for every symbol in $S$ we include a corresponding factor $p$, and for every symbol not in $S$ we include a corresponding factor $1-p$). If the answer to our query is \emph{yes} then we stay at $x$, and if the answer is \emph{no} then we go to $y$. Note that the marginal probability that $\text{``$<$''} \in S$ is $p_\ell$, so the probability to stay in $x$ conditioned on $f(x)<f(y)$ is also $p_\ell$, and similar for ``$=$'' and ``$>$''. Hence, by an easy case distinction on whether $f(x)$ is less, equal, or larger than $f(y)$, we find that in all cases the probability of going to $y$ is the same as for the algorithm $A$. Thus we can simulate $A$ in the model $\A'$.

It remains to prove a lower bound on the $\A'$-complexity of $\F$. By Yao's Principle, it suffices to prove such a bound for the expected runtime of every deterministic algorithm $A\in \A'$ on a randomly chosen function. We regard a distribution on $\F$ where for each $z\in\{0,1\}^n$ exactly one function with optimum $z$ has probability $2^{-n}$ to be drawn, and all other functions in $\F$ have zero probability. Note that the $\A'$-oracle gives only two possible answers (one bit of information) to each query. By a standard information-theoretic argument~\cite{DrosteJW06} we show that the probability that the $i$-th query of $A$ is the optimum is at most $2^{-n+i-1}$. More precisely, observe that after $i-1$ queries we can distinguish at most $2^{i-1}$ cases so that on average $2^{n-(i-1)}$ search points are still possible optima. By the choice of our distribution, each one of them is equally likely to be the optimum of function $f$. Let $c_j$ be the number of search points that are still possible in the $j$-th case, for $1\leq j\leq 2^{i-1}$. Then the probability of hitting the optimum is 
\[
\sum_{\substack{j \in [2^{i-1}]\\ \Pr[\text{case $j$}] >0}} c_j^{-1}\cdot \Pr[\text{case $j$}] = \sum_{\substack{j \in [2^{i-1}]\\ \Pr[\text{case $j$}] >0}} 2^{-n} \leq 2^{-n+i-1}.
\]

Furthermore, by the union bound the probability that the optimum is among the fist $i$ queries is at most $\sum_{j=1}^{i}{2^{-n+i-1}} < 2^{i-n}$. This immediately implies the statement on the Monte Carlo complexity. For the other complexities, the claim follows by observing that the number $T$ of queries to find the optimum has expectation
\begin{align*}
\E[T] 
& \geq \sum_{i=0}^{n-1} \Pr[T > i] 
\geq \sum_{i=0}^{n-1}(1-2^{i-n})\\
& = n - 2^{-n}\sum_{i=0}^{n-1}{2^i} 
= n - (1-2^{-n}) 
\geq n-1.
\end{align*}

For the $(1+\lambda)$-case with $\lambda \geq 1$ we consider the following set $\A'$ of algorithms, which have access to their complete search history. We require the algorithm to partition the set of weak orderings of $\lambda+1$ elements (i.e., orderings with potentially equal elements) into $\lambda+1$ subsets $S_1,\ldots,S_{\lambda+1}$, and the oracle tells the algorithm to which subset the ordering of the fitnesses of the $\lambda+1$ search points belongs. Then each $(1+\lambda)$ memory-restricted ranking-based black-box algorithm $A$ can be simulated in this model. More precisely, fix $\lambda+1$ search points $y_1,\ldots,y_{\lambda+1}$ (where $y_1$ is the parent individual and $y_2,\ldots,y_{\lambda+1}$ are the offspring). Then for each weak ordering $\sigma$ of these search points, let $p_1(\sigma),\ldots,p_{\lambda+1}(\sigma)$ be the probability that $A$ selects the first, second, $\ldots, \lambda+1$-st search point, respectively, if they are ordered according to $\sigma$. Then we can simulate $A$ by choosing a partitioning $(S_1,\ldots,S_{\lambda+1})$ with probability 
\[
p(S_1,\ldots,S_{\lambda+1}) = \prod_{i=1}^{\lambda+1}\prod_{\sigma \in S_i} p_i(\sigma).
\]
In this way, for every ordering $\sigma$ of the $\lambda+1$ search points and for every $1\leq i \leq \lambda+1$ the marginal probability that $\sigma \in S_i$ is $p_i(\sigma)$. Thus, if the oracle tells us that the ordering of the search points is in the $i$-th partition then we select the search point $y_{i}$. In this way, we have the same probability $p_i(\sigma)$ as $A$ to proceed to search point $y_i$. Hence, we can simulate $A$ in this model.
 
In order to prove a lower bound for $\A'$ we employ Yao's principle as for the (1+1) case. Note that the algorithm learns $\log(\lambda+1)$ bits per query. Similarly as before, the probability that the $i$-th query of a deterministic algorithm is the optimum is at most $(\lambda+1)^{i-1}2^{-n}$, and a similar calculation as before shows that $\Pr[T \leq i] \leq (\lambda+1)^{i}2^{-n}$ and $\E[T] \geq n/\! \log (\lambda+1) -O(1)$.

For $\mu > 1$ and $\lambda =1$ we learn the position of the new search point among the $\mu$ previous search points. There are at most $2\mu+1$ positions for the new search point (its fitness may equal the fitness of one of the other search points, or it may lie between them). Thus we only learn at most $\log(2\mu+1)$ bits of information per query, and we can derive the complexities in the same manner as before.

If both $\mu$ and $\lambda$ are larger than $1$, then there are at most $\binom{\mu+\lambda}{\mu}$ ways to select $\mu$ out of $\mu+\lambda$ search points, and there are $B_{\mu} = (1+o(1))\mu !(\ln 2)^{-\mu} /2$ weak orderings on these $\mu$ elements (i.e., orderings with potentially equal elements), where $B_\mu$ is the $\mu$-th ordered Bell number~\cite{Gross1962}. Hence, the algorithm can learn at most $b:=\log((1+o(1))\binom{\mu+\lambda}{\mu} \mu !(\ln 2)^{-\mu} /2)$ bits per query. Since $b = \log(\binom{\mu+\lambda}{\mu}) + \mu (\log \mu -1 - \log \ln 2)-\log 2 + o(1)$, this implies the claim in the same way as before.
\end{proof}

Note that the lower bounds given by Theorem~\ref{thm:lower} are by a constant factor stronger than the lower bounds for general comparison-based algorithms (that are not memory-restricted), if they learn all comparisons among the $\mu+\lambda$ search points. For example, in the classical case where we may compare exactly two search points (corresponding to the (1+1) case), we only get a lower bound of $n/\! \log(3) -O(1)$ instead of $n-1$. Intuitively speaking, the reason is that a comparison-based algorithm may use the three possible outcomes ``larger'', ``less'', or ``equal'' of a comparison, while memory-restricted comparison-based algorithms only get two outcomes ``stay at $x$'' or ``advance to $y$''.

We remark that the analysis for $\mu>1$ can be tightened in several ways. Firstly, for the elitist $(\mu+1)$ black-box complexity, we only have $2\mu$ cases instead of $2\mu+1$ since we can -- sloppily speaking -- not distinguish between the case that the new search point is discarded because it has worse fitness than the worst of the $\mu$ old ones, or whether it is discarded because it has equal fitness to the worst of the $\mu$ old search points. Moreover, for all black-box models under consideration we learn $\log(2\mu+1)$ bits of information in the $i$-th round only if all previous search points have different fitnesses; otherwise, we get less information. However, if the new search point has fitness equal to one of the old fitnesses, then with the next query we get less information. Also for the case $\mu >1$ \emph{and} $\lambda >1$ the bound in Theorem~\ref{thm:lower} can be tightened at the cost of a more technical argument. 

\section{The (2+1) Elitist Black-Box Complexity of OneMax}
\label{sec:21upper}

For the $(2+1)$ elitist black-box complexity, a simple algorithm proves to have complexity at most $n+1$. The algorithm is deterministic, so it provides an upper bound to both the Monte Carlo complexity and the Las Vegas complexity.

\begin{theorem}
\label{thm:21upper}
The (Monte Carlo and Las Vegas) (2+1) elitist black-box complexity of \onemax is at most $n+1$. 
\end{theorem}

\begin{proof}
Throughout the algorithm, we maintain the invariant that in the $i$-th step we have two strings $x_i$ and $x_i'$ that are both optimal in the first $i$ bits, that are both zero on bits $i+2,\ldots,n$ and that differ on bit $i+1$ (one of them is $0$, the other is $1$). 

We thus start with the all-zero string $x_0=(0,\ldots,0)$ and the string $x_0'=(1,0,\ldots,0)$. Given $x_i$ and $x_i'$, take the string with the smaller fitness (say $x_i'$), and flip both the $i$-th and the $(i+1)$-st bit in it, giving a string $x_{i+1}'$. (The index $i$ is determined by $x_i$ and $x_i'$.) Since the $i$-th bit in $x_i'$ was incorrect, the fitness of $x_{i+1}'$ is at least as high as the fitness of $x_i'$ and we may thus replace $x_i'$ by $x_{i+1}'$. 
The invariant is maintained with $x_{i+1} = x_i$, since both $x_{i+1}$ and $x_{i+1}'$ are optimal on the $i$-th bit (and on all previous bits by induction). In this way, the $n$-th generation will contain an optimal search string, and at most $n+1$ fitness evaluations are needed in these $n$ generations.
\end{proof}

\section{The (1+1) Elitist Black-Box Complexity of OneMax}
\label{sec:11upper}

We start with a high level overview of the algorithm in Section~\ref{sec:11upperOverview}. Some tools needed for its runtime analysis are presented in Section~\ref{sec:toolsupper}, while the formal analysis is carried out in Section~\ref{sec:11upperProof}. 

\subsection{Overview}
\label{sec:11upperOverview}

While the algorithm and the analysis in Section~\ref{sec:21upper} are rather straightforward, the analysis for the (1+1) situation is considerably more difficult. In fact, we do not know the Las Vegas black-box complexity of \onemax for (1+1) elitist algorithms. As we shall discuss below, if we only had one additional bit that we could manipulate in an arbitrary way, we could show that it is of linear order, but we do not know how to create such a bit. Still, the general ideas for that algorithm show a linear Monte Carlo black-box complexity. According to the lower bound (Theorem~\ref{thm:lower}), this  is best possible. 

\begin{theorem}
\label{thm:11upper}
The Monte Carlo (1+1) elitist black-box complexity of \onemax is $\Theta(n)$.
\end{theorem}

The lower bound in Theorem~\ref{thm:11upper} follows from Theorem~\ref{thm:lower}. We thus concentrate in the following on the upper bound. 
As in previous works on black-box complexities for \onemax, in particular the memory-restricted algorithm from~\cite{DoerrW14memory}, we will use some parts of the bit string for storing information about the search history.
 
The main idea of the algorithm is similar to the one of the previous section. That is, we aim at optimizing one bit at a time. Since we cannot encode any more the current iteration in the population, we implement instead a counter which tells us which bit is to be tested next. The main difficulty is in \emph{(i)} designing a counter that does not affect the fitness of the bit string, and \emph{(ii)} optimizing a bit with certainty in constant time. As we shall see in Section~\ref{sec:toolsupper}, a counter can be implemented reserving $O(\log n)$ bits of the string exclusively for this counter, solving (i). Point (ii) can be solved if we may access a small pool of non-optimal bits (which we call \emph{trading bits}). The key idea is that throughout the algorithm in expectation we gain more trading bits than we spend, so we never run out of trading bits. 

The main steps of the algorithm verifying Theorem~\ref{thm:11upper} are thus as follows.
\begin{enumerate}
	\item Create a neutral counter for counting numbers from $1$ to $n$.
	\item Create a pool of $\omega(\log n)$ trading bits, all of which are non-optimal.
	\item Using the trading bits, optimize the remaining string (the part unaffected by the counter) by testing one bit after the other. Use the counter to indicate which bit to test next. At the same time, try to recover trading bits if possible.
	\item Using RLS (Algorithm~\ref{alg:RLS}), optimize the part which had been used as a counter. We use bit $b_0$ as a flag bit to indicate that we are in Step 4.
	\item Optimize $b_0$.
\end{enumerate}
The technically most challenging parts are Step 1 and Step~3. But, interestingly enough, the key problem in turning the Monte Carlo algorithm into a Las Vegas one lies in separating Step 5 from Step 4: we need to test every once in a while during the fourth phase whether or not bit $b_0$ is optimal. If we test too early, that is, before Step 4 is finished, it may happen that we have to accept this offspring and thus misleadingly assume that we are in one of the first three steps, yielding the algorithm to fail.
Note though that this problem could be completely ignored if we had just one bit that we could manipulate as we want (i.e., without having to use elitist operations). 

Due to all the necessary preparation, the formal proof of Theorem~\ref{thm:11upper} will be postponed to Section~\ref{sec:11upperProof}.

\subsection{Tools for Proving Upper Bounds}
\label{sec:toolsupper}
 
In this section we collect tools that are common in the algorithms of the subsequent sections. All the following operations will be Monte Carlo operations, i.e., they have some probability of failure. Recall from Remark~\ref{rem:1} that if we have an algorithm $A$ for a set of functions $\F$ and a ``failure event'' $\calE_{\text{fail}}$ of probability $p_{\text{fail}}$ such that conditioned on $\neg \calE_{\text{fail}}$ the algorithm $A$ succeeds after time $T$ with probability at least $1-(p-p_{\text{fail}})$, then the $p$-Monte Carlo complexity of $A$ on $\F$ is at most $T$. In particular, if conditioned on $\neg \calE_{\text{fail}}$ the algorithm $A$ succeeds after \emph{expected} time $T$, then by Markov's inequality it succeeds after time at most $(p - p_{\text{fail}})^{-1}T$ with probability at least $1-(p - p_{\text{fail}})$. Therefore, the $p$-Monte Carlo complexity of $A$ on $\F$ is at most $O(T)$ for any $p > p_{\text{fail}}$ with $p-p_{\text{fail}}\in \Omega(1)$. 

\subsubsection{Copying or Overwriting Parts of the String}\label{sec:copy}
Our first operation will be a copy operation. If we have a large part $B$ of the string with a constant fraction of non-optimal bits, then we can efficiently copy a small substring into a new position by flipping some non-optimal bits of $B$. After the operation, $B$ is still of a form that may be used for further copy operations, except that the number of non-optimal bits in $B$ has decreased. Note that a string drawn uniformly at random of, say, length $n/2$ may serve as $B$ since with high probability roughly half of the bits in $B$ will be correct.

\begin{lemma}
\label{lem:copy}
Assume we have a set $B$ of $b$ known bit positions, of which at least $b_0 = \beta b$ bits are non-optimal, for some $\beta>0$, and the position of the non-optimal bits are uniformly at random in $B$. 
Assume further that we have two sets $C,C'$ of bit positions such that $|C| = |C'| \leq b_0/2$, and that $B,C,C'$ are pairwise disjoint. 

There is a (1+1) elitist black-box strategy that copies the bits from $C$ into $C'$. 
For any $c>0$ this algorithm requires 
at most $c\cdot|C|\cdot\log(n)/\beta$ iterations with probability $1-n^{-\Omega(c)}$. 
After the copy operation, at least $b_0-|C|$ bits in $B$ will be non-optimal, and their positions will be uniformly at random in $B$.

The same strategy can be used to overwrite $C'$ with a fixed string (e.g., $(1, \ldots, 1)$).
\end{lemma}

\begin{proof}
We perform the following operation until $C$ and $C'$ are equal. Assume that in the current search point $x$ the first $i-1$ bits of $C$ and $C'$ coincide, but the $i$-th bits differ, for some $i>0$. Sample a new search point $x'$ by flipping the $i$-th bit of $C'$ and a random bit in $B$. Accept $x'$ if $f(x')\geq f(x)$. Note that we flip exactly one bit in $B$ for every bit in $C$ that we copy, so at all times at least $b_0 - |C| \geq b_0/2$ bits in $B$ are non-optimal. 

If the $i$-th bit of $C'$ was non-optimal, we accept $x'$ in any case. Otherwise, we accept it if the random bit in $B$ was non-optimal, which happens with probability at least $b_0/(2b) = \beta/2$. Thus we need in expectation at most $2/\beta$ trials to copy the $i$-th bit, proving that the expected runtime is at most $O(|C|/\beta)$. By the Chernoff bounds, the runtime is more than $c|C|\log(n)/\beta$ with probability at most $n^{-\Omega(c)}$. Finally, since we choose the bits in $B$ uniformly at random, the positions of the non-optimal bits in $B$ are uniformly at random after each step.
\end{proof}

\subsubsection{Reliable Optimization} 

In this section we give a routine that allows to be sure with very high probability that some small part of the string is optimal.

\begin{lemma}
\label{lem:optimize}
For every $0<p<1/2$, $p = e^{-o(n)}$, there is $\ell= O(\log(1/p))$ such that the following holds for all $\beta >0$ and $k\in \N$. Let $x$ be a bit string in $\{0,1\}^n$ such that $x_1=x_2 = \ldots = x_{\ell}=0$, and assume that in the remaining string there is a block $B$ of known position of size at least $2\ell/\beta$ such that at least a $\beta$ fraction of the bits in $B$ are non-optimal, their positions distributed uniformly at random in $B$. Moreover, let $C$ be a block of size $k$ that is disjoint of $x_1,\ldots,x_\ell$ and of $B$. 

Then there exists a (1+1) elitist black-box strategy that with probability at least $1-p$ optimizes $C$ in time $\mathord{O}\mathord{\left(\ell k \log k/\beta\right)}$ and marks termination by setting $x_\ell$ to $1$. The algorithm will optimize at most $\ell$ random bits of $B$ by copy operations as in Lemma~\ref{lem:copy}.

\end{lemma}

Note that Lemma~\ref{lem:optimize} can be achieved with trivial algorithms (e.g., RLS) if we do not insist that the algorithm marks termination. This is an important part since knowing when a phase has finished will be a crucial ingredient for further algorithm. We remark that the bits $x_1,\ldots,x_\ell$ in Lemma~\ref{lem:optimize} may be replaced by any bits as long as the positions are known. We remark further that the requirement $p = e^{-o(n)}$ can be replaced by $p=\Omega(e^{-cn})$ for some suitably chosen constant $c>0$. For our purposes the claimed setting suffices.

\begin{proof}
We will use the number of one-bits among $x_1,\ldots,x_{\ell}$ as an estimator for the time that we have already spent. In each step with probability $1-1/(3 k \log k)$ we use an RLS step (Randomized Local Search, Algorithm~\ref{alg:RLS}) on $C$. Otherwise, we flip the first of the bits $x_1,\ldots,x_{\ell}$ that is still zero, and we flip simultaneously a random bit in $B$. 

By our assumption made above, at most $\ell$ bits of $B$ will be flipped, so during the the whole algorithm each one of them is non-optimal with probability at least $(\beta|B|-\ell)/|B| \geq \beta/2$. Thus in each step we successfully flip one of the bits $x_1,\ldots,x_{\ell}$ with probability at least $p':=\beta/(2 k \log k)$. By the Chernoff bound, the probability that after $n'= 4\ell k\log k/\beta$ steps we have not flipped all of them is at most
\[
\Pr[\text{Bin}(n',p') \leq \ell] \leq e^{-\Omega(\ell)} \leq p/3
\]
for a suitable choice of $\ell = O(\log (1/p))$. Thus the algorithm terminates after at most $n'$ steps with probability $1-p/3$.

Let us split the execution into $\ell$ rounds, where the $i$-th round is characterized by $x_1=\ldots =x_{i-1}=1$ and $x_i=\ldots=x_{\ell} =0$. Let us call a round which takes at least $2k\log k$ RLS steps on $C$ a \emph{long round}. Since the number of RLS steps in each round is geometrically distributed, a round is long with probability at least 
\begin{align*}
\left(1-\frac{1}{3k \log k} \right)^{2 k \log k} \geq \left(1-\frac{1}{3}\right)^2
=
\frac{4}{9},
\end{align*}
since the function $(1-1/x)^x$ is monotonically increasing in $x\geq 1$. Thus, by the Chernoff bound there are at least $2\ell/9$ long rounds with probability at least $1-e^{-\Omega(\ell)} \geq  1-p/3$ for a suitable choice of $\ell$. On the other hand, the probability that $C$ is not optimized in a long round is at most $1/k$ (this is an application of the coupon collector problem, see~\cite[Theorem 1.23]{Doerr11bookchapter}). So the probability that $C$ is not optimized by any of $\Omega(\ell)$ long rounds is at most $e^{-\Omega(\ell)}\leq p/3$ for a suitable choice of $\ell$. Summarizing, with probability $1-p$, the algorithm succeeds in time at most $O(\ell k \log k/\beta)$.
 \end{proof}

\subsubsection{A Neutral Counter} \label{sec:counter}
Next we show that it is possible to set up a counter in a way that increasing the counter does not affect the \onemax values of the string. The counter can be implemented in the (1+1) elitist black-box model, and is hence applicable in any $(\mu+\lambda)$ elitist black-box model. 

\begin{lemma}[Neutral Counter]
\label{lem:11counter}
For every $0<p<1/2$, $p = e^{-o(n)}$, there is $\ell= O(\log(1/p))$ such that the following holds. Let $x$ be a bit string in $\{0,1\}^n$ such that $x_1=x_2 = \ldots = x_{\ell+2}=0$ and $(x_{\ell+3}, \ldots, x_n)$ is uniformly distributed in $\{0,1\}^{n-\ell+3}$. 

Then there exists a (1+1) elitist black-box strategy that with probability at least $1-p$ implements in $x$ a counter which can be used during future iterations without changing the \onemax value of the string. 
For counting up to $j = O(n)$, the counter requires a total number of $O(\log j)$ bits that are blocked in all iterations in which the counter is active. 
The setup of the counter requires 
$\mathord{O}\mathord{\left(\ell \log j \log\log j\right)}$ function evaluations. 
During the setup of the counter, 
$O(\log j + \ell)$ random bits of the remainder of the string are optimized by copy operations as in Lemma~\ref{lem:copy}.
\end{lemma}

\begin{proof}
In all that follows we use a partition $C$, $C'$, and $B$ of $[n]\setminus[\ell+2]$, thus splitting the string $x$ (minus the first $\ell+2$ bits) into three parts, which by abuse of notation we also call $C$, $C'$, and $B$. The sizes of $C$ and $C'$ are $k = O(\log n)$ each (see below), so the size of $B$ is $n-o(n)$. By assumption, the entries in $B$ are initialized uniformly at random. Note that by the Chernoff bound, with exponentially high probability at least a $1/3$ fraction of the bits in $B$ are non-optimal. We will henceforth assume that this is the case (giving the algorithm a failure probability of $e^{-\Omega(n)} \leq p/2$). As we will see, the counter algorithm will use at most $2|C|+\ell+2$ ``payoff bits'' which are flipped from a non-optimal into the correct state, so that at any point time during the algorithm at least a $1/3 -o(1)$ fraction of the bits in $B$ will be non-optimal.

Let $k$ be the smallest even integer such that $\binom{k}{k/2} \geq j$. 
Note that $k= O(\log j)$, since $\binom{k}{k/2}\geq 2^k/\sqrt{\pi k}$ by Stirling's formula. We use Lemma~\ref{lem:optimize} to optimize block $C$ with probability $1-p/2$, using $x_3,\ldots,x_{\ell+2}$ as flag bits. 

Once we have a string in the memory which satisfies $x_1=x_2=0$ and $x_{\ell+2}=1$, we assume that part $C$ is optimized and we copy the entries of $C$ into part $C'$ using Lemma~\ref{lem:copy} with part $B$ as payoff bits. As soon as $C$  has been copied into $C'$, we want to change the second flag bit. We do this by flipping $x_2$ plus a random bit in $B$ until the corresponding string is accepted. The flag $011$ in the first three positions tells us to move on to initializing the counter.

We fix an enumeration of all the $\binom{k}{k/2}$ possible ways to set exactly $k/2$ out of the $k$ entries to their correct values. Let $r_1, \ldots, r_j$ be the first $j$ strings corresponding to this enumeration. For initializing the counter to one we copy the string $r_1$ into $C$ by applying Lemma~\ref{lem:copy}, again with part $B$ as payoff bits. When we have initialized the counter, we finally flip the flag bit $x_1$ (together with a random bit in $B$) to indicate that the counter is ready. 

Note that throughout the whole algorithm, we use at most $2|C| +\ell+2$ payoff bits, as we claimed at the beginning of the proof. Note also that if at least $n/3$ of the bits in $B$ are non-optimal, the second and the third phase are Las Vegas operations (they can never fail, but the time needed for these phases is random).

Since the optimal entries of $C$ are stored in $C'$ (the bits in $C'$ will not be touched as long as the counter is active), we can at any time read the value of the counter by comparing $C$ with $C'$.  
Similarly, if we want to increase the counter from some value $i$ to $i+1$, we flip simultaneously those bits of $C$ in which $r_i$ and $r_{i+1}$ differ. Since there are exactly $k/2$ ones in either of the two strings $r_i$ and $r_{i+1}$, this does not affect the \onemax-value of the string.  
\end{proof}

\subsubsection{Optimizing in Linear Time with Non-Optimal Bits}
\label{sec:tradingbits}
The following lemma allows us to optimize a large part of the string in linear time, provided that we have some small area $B'$ with ``trading bits'', i.e., with bits that are non-optimal.

\begin{lemma}
\label{lem:tradingbits}
Let $0 <\alpha < 1$ be constant. Assume we have two counters $C,C'$ that can count up to $n$ and a flag bit $b$ that is set to $0$. Assume further that we have two blocks $B,B'$, with $|B|, |B'| = \omega(\log n)$ such that all bits in $B'$ are non-optimal, and that at least an $\alpha$ fraction of the bits in $B$ is non-optimal, their positions distributed uniformly at random. Then there is a (1+1) elitist black-box algorithm that optimizes $B$ and $B'$ in linear time with probability $1-o(1/n)$. 
\end{lemma}

\begin{proof}
We start with the counters $C$ and $C'$ at $0$, and go through the bits in $B$ one by one, maintaining the following invariants. When $C$ is at $i$ then the first $i$ bits in $B$ are optimal. When $C'$ is at $i'$ then the first $i'$ bits in $B'$ are optimal, and all further bits in $B'$ are non-optimal. We will call the non-optimal bits in $B'$ \emph{trading bits}.

Choose $0<p<1$ so small that $2p/(1-p)-\alpha < 0$. Assume $C$ is at position $i$ and $C'$ at position $i'$. If $i'=0$ then we simply flip the first bit in $B'$ and increase $C'$, so assume $i'>0$ from now on. In each step we flip a coin. If it turns head (with probability $1-p$), then we flip the $i+1$-st bit in $B$, increase $C$, flip the $i'$-th bit of $B'$ and decrease $C'$. If the offspring has equal fitness, we accept it. Note that the fitness is equal if and only if the bit in $B$ was non-optimal, and that we recover one of the non-optimal trading bits in $B'$ in this case. On the other hand, if the bit in $B$ was optimal in the original string then the fitness of the new search point is strictly smaller than the previous one so that the offspring is immediately discarded. So we only accept an increase in $C$ if the $i$-th of $B$ is correct in the new string. 
If the coin flip was tail (with probability $p$) then we just flip the $i+1$-st bit in $B$, flip the $i'+1$-st bit in $B'$, and increase $C'$ (but do not touch $C$). Note that we may (and will) accept the offspring in any case. 

Evidently, we maintain the invariant mentioned above. Moreover, we spend only an expected constant number of iterations for optimizing a bit in $B$, and by the Chernoff bound the algorithm optimizes $B$ in at most $c|B|$ with probability $1-e^{-\Omega(|B|)} = 1-o(1/n)$, for a suitable $c>0$. Once $B$ is optimized (i.e., once the counter $C$ is at position $|B|$), we flip all non-optimal bits in $B'$ in one step. The only way the algorithm can fail (except by taking too long, which only happens with probability $1-o(1/n)$) is by running out of trading bits, so it remains to show that with high probability this does not happen.

Let $X_i$ the number of trading bits that are used up after the $i$-th round, and let $\Delta_i := X_{i+1}-X_{i}$ be the number of trading bits that we spend in this round. For the sake of exposition, assume first that there is an unlimited number of trading bits that can be gained or used in this round (while in fact, the total number of trading bits must stay between $0$ and $B'$). If the $i$-th bit of $B$ was optimal then the algorithm waits for tails to proceed. This costs us one trading bit and brings us into the position that the $i$-th bit of $B$ is non-optimal. In that position, we either (with probability $1-p$) proceed to the $i+1$-st bit and gain a trading bit, or we proceed (with probability $p$) to the other position, pay a trading bit, and pay another one to return to the old position. So if the $i$-th bit is initially non-optimal, then the expected number of trading bits that we spend for optimizing the $i$-th bit is
\begin{align*}
& \E[\Delta_i \mid \text{$i$-th bit non-optimal}]  = -(1-p)+p(2+\E[\Delta_i \mid \text{$i$-th bit non-optimal}]),
\end{align*}
from which we easily deduce $\E[\Delta_i \mid \text{$i$-th bit non-optimal}] = -1+2p/(1-p) $ and $\E[\Delta_i \mid \text{$i$-th bit optimal}] = 1+\E[\Delta_i \mid \text{$i$-th bit non-optimal}] = 2p/(1-p) $. The probability that the $i$-th bit is non-optimal is at least $\alpha$, and so 
\[
\E[\Delta_i] < \alpha \left(-1+\frac{2p}{1-p}\right)+(1-\alpha)\frac{2p}{1-p} = \frac{2p}{1-p}-\alpha< 0.
\] 
Now we examine how the drift changes if the number $X_i$ of trading bits that we can gain in the $i$-th round is bounded. Since the probability to spend more than $b$ trading bits in one round goes (geometrically) to zero as $b\to \infty$, there is a constant $b_0>0$ and a constant $\eps>0$ such that 
\[
\E[\Delta_i \mid X_i \geq b_0 ] \leq -\eps
\]
Therefore, the number $X_i$ of used trading bits performs a random walk with negative drift while it is between $b_0$ and $|B'|$. Moreover, the probability $\Pr[|\Delta_i| \geq j]$ decreases geometrically in $j$. Therefore, by the Negative Drift Theorem~\ref{thm:negativedrift} (with constant $r(\ell)$ and $\ell = |B'|-1-b_0 = \omega(\log n)$) the probability that any of the $X_i$ exceeds $|B'|-1$ for $1\leq i\leq |B|$ is at most $e^{-\Omega(|B'|-1-b_0)} =o(1/n)$, so we are not out of trading bits \emph{after} any round. But in every round we gain at most one trading bit, so if $X_i$ does not exceed $|B'|-1$ then at no point during the $i$-th round the algorithm uses more than $|B'|$ trading bits. This proves that we never run out of trading bits with probability $1-o(1/n)$.
\end{proof}

We remark without proof that Lemma~\ref{lem:tradingbits} can be strengthened to hold with probability $1-p$ for any $p=e^{-o(n)}$ if $|B|, |B'| = \omega(\log(1/p))$.

\subsection{Proof of Theorem~\ref{thm:11upper}}
\label{sec:11upperProof}

\begin{proof}
We split the string into four parts: firstly a constant number of flag bits indicating in which phase of the algorithm we are. Some of them we use for the subroutines, but bit $b_0$ is kept to be $0$ until the very last phase. Then two counters $C,C'$ that can count up to $n$. Further, we have a part $B'$ of the string of size $O(\log^2 n)$ which we use as trading bits, and the remaining part $B$.

Now we put all pieces together. We initialize the flag bits as $0$, and initialize $B$ uniformly at random. Then we build the counters as described in Lemma~\ref{lem:11counter}, using the randomness from $B$, and indicate with a flag bit when we are finished. We split $B'$ into two parts $B_1'$ and $B_2'$ of equal size. Then we use Lemma~\ref{lem:optimize} to optimize $B_1'$ with high probability, setting a flag bit when finished.  
When this flag is set, we copy $B_1'$ into $B_2'$, and then we copy the string $B_2' \oplus (1,\ldots,1)$ into $B_1'$, effectively inverting all bits in $B_1'$. For both copy operations we use the randomness from $B$. Note that afterwards we still have an $1/2-\eps$ fraction of non-optimal bits in $B$ (using Chernoff bounds and the fact that all copy operations together touch $o(n)$ bits), and that all bits in $B_1'$ are non-optimal. Thus we can apply Lemma~\ref{lem:tradingbits} to optimize $B, B_1'$ and $B_2'$ in linear time with probability $1-o(1/n)$. In the last step of this phase, we also set $x_0$ to $1$. (We can do this since the last operation flips all the non-optimal bits in $B_1'$).

While $x_0$ is $1$, we do the following. With probability $1-\ln n/n$ we flip a bit outside of $B \cup \{x_0\}$. With probability $\ln n/n$ we flip $x_0$. Note that there are at most $O(\log^2 n)$ bits outside of $B$, so this region (except of $x_0$) will be optimized after an expected number of $O(\log^2 n \log \log n)$ steps. Moreover, by the Coupon Collector Theorem the probability that it takes more than $c\log^2n \log \log n$ steps to optimize $B$ is at most $1/n$ for at suitable constant $c\geq 1$~\cite[Theorem 1.23]{Doerr11bookchapter}. Hence, when $x_0$ is flipped, then with probability at least $1-\log^3n/(n\log \log n)$ we have found the optimum.  On the other hand, with probability $1-(1-\ln n/n)^n\geq 1-1/n$ this phase takes at most $n$ steps. This concludes the proof.
\end{proof}

\begin{remark}\label{rem:iteratedcounter}
In the proof of Theorem~\ref{thm:1lambdaupper} the main part is Las Vegas. We have failure probabilities only for initializing the counter, by running out of trading bits, and for the optimization of the bits which are reserved for the counter. The first two failure probabilities can be made superpolynomially (in fact, exponentially) small. The failure probability stemming from the last phase can be decreased by using an \emph{iterated counter}, which is used to reduce the number of bits that are blocked for the operation of the counter. 

More precisely, we start with the counter described above, which can count from one to $j_1=j$. A second counter is implemented, again using Lemma~\ref{lem:11counter}, to count from $1$ to $j_2 = \Theta(\log j_1)$, a third one for counting from one to $j_3 = \Theta(j_2)$, and so on until the size of the bits that need to be blocked for the counter is at most constant. Then we optimize the main part of the string as before, but making sure that $j_1+j_2+\ldots = O(\log^2 n)$ bits remain that are all non-optimal. With these bits, we can optimize the region of the first counter without error probability, using the second counter and $j_1$ of the non-optimal bits. Then we optimize the region of the second counter using the third counter, and so on until we end up with a counter that has only constant size. This counter we then optimize with RLS steps as described in the proof. Effectively, this allows us to design an algorithm (by flipping the last bit with probability $\ln n/n$) that needs time $O(n)$ with probability $1-O(\log n/n)$. 

Alternatively, although it gives neither a Monte Carlo nor a Las Vegas complexity, note that there is an algorithm (by flipping the last bit with probability $1/n$) for which there is an event $\mathcal{E}_{\text{bad}}$ of Probability $\Pr[\mathcal{E}_{\text{bad}}]= O(1/n)$ (namely, the event that either initialization fails, or that the last bit is flipped too early) such that conditioned on $\neg \mathcal{E}_{\text{bad}}$ the algorithm has an expected runtime of $O(n)$.
\end{remark}

\begin{remark}
Note that if we had only one bit of additional memory, then we could use it as an indicator bit for random local search: in any step of the algorithm, we could with some small probability (e.g., with probability $1/(n\log n)$) flip this bit, and then proceed with random single bit flips from this point on. If the success probability of the Monte Carlo algorithm is at least $1-O(1/\log n)$ (we proved much stronger bounds), then this results in a Las Vegas algorithm with linear expected runtime. Unfortunately, it is unclear how to make use of high success probabilities without an additional bit of memory, so our results do not imply a linear Las Vegas runtime. 
\end{remark}

\section{The \texorpdfstring{$(1+\lambda)$}{1+l} Elitist Black-Box Complexities of OneMax}
\label{sec:1lambdaupper}

We have already seen in Section~\ref{sec:21upper} that a slight increase of the population size of the elitist black-box model can significantly simplify the \onemax problem. 
In the (2+1) model considered in Section~\ref{sec:21upper} we were in the comfortable situation that we could use the two strings of the memory to encode an iteration counter. 
In this section we regard the $(1+\lambda)$ elitist black-box model. Intuitively, this model is less powerful than the $(\lambda+1)$ model since we have to base our sampling strategies solely on the one search point in the memory. Still the model allows to check and compare several alternatives at the same time, so it should be considerably easier than the (1+1) situation. 
The core idea of the following theorem is to divide the bit string into blocks of size $\log \lambda$ each and to optimize these blocks iteratively by exhaustive search. 

\begin{theorem}
\label{thm:1lambdaupper}
Let $\eps, C >0$, and let $1<\lambda< 2^{n^{1-\eps}}$. For suitable $p = O(\log^2 n\log \log n\log \lambda/n)$ there exists a $(1+\lambda)$ $p$-Monte Carlo elitist black-box algorithm that needs at most 
$O(n/\! \log \lambda)$ generations on \onemax.
\end{theorem}

We emphasize that the bound in Theorem~\ref{thm:1lambdaupper} is in number of generations, not in terms of function evaluations. We feel that this is the more useful measure, in particular when the $\lambda$ offspring can be generated in parallel. Note that an algorithm optimizing for the number of function evaluations can be substantially different from the ones minimizing the number of generations. 

\begin{proof}[Proof of Theorem~\ref{thm:1lambdaupper}]
We initialize the algorithm by implementing the (1+1) elitist counter as described in Lemma~\ref{lem:11counter}. 
We split the remaining string into blocks of length at most $\lfloor \log_2 \lambda \rfloor \geq 1$, and we want to optimize each block with exhaustive search. There are $j=\lceil n/ \lfloor \log_2 \lambda \rfloor \rceil = \Omega(n^{\eps})$ such blocks, and we thus apply Lemma~\ref{lem:11counter} with this $j$. 
This requires $O(\log (1/p) \log j(\log\log j)^2) = O(\log^3 n)$ generations. 
The counter blocks $O(\log j)$ bits which we cannot touch during the optimization of the blocks (except, of course, for operating the counter).

We then optimize the $n-O(\log j)$ bits which are not blocked for the counter. We optimize $\lfloor \log_2 \lambda \rfloor$ bits in each iteration, by sampling all possible $2^{\lfloor \log_2 \lambda \rfloor} \leq \lambda$ entries in the block. In each sampled offspring the counter is increased by one (when compared with the counter of the parent individual). In the last generation we possibly optimize a block that is smaller than $\lfloor \log_2 \lambda \rfloor$, but the routine is the same, i.e., exhaustive search. Note that this optimization routine is deterministic. It requires at most $j$ generations.

Once the counter of the parent individual shows $j$ we need to optimize those bits that were reserved for the counter. We do this in the same way as we did in the (1+1) situation (see Section~\ref{sec:11upperProof}). That is, we use Randomized Local Search (RLS, Algorithm~\ref{alg:RLS}) on the yet unoptimized part and with some probability $p'=O(\log n\log\lambda/n)$ we flip the bit $b_0$ indicating us to do RLS steps. At the time that bit $b_0$ is flipped for the first time, the remainder of the bit string is optimized with probability at least $1-p/2$ for a suitable choice of $p'$, and the probability that it needs more than $C n/\! \log \lambda$ steps is $O(1/n) = o(p)$ for a suitable choice of $C>0$.
\end{proof}

We remark without formal proof that the requirement on $p$ can be relaxed by regarding an iterated counter (cf.\! Remark~\ref{rem:iteratedcounter}). If $\lambda$ is a small constant, then we may use $p=O(\log n/n)$ as in the (1+1) case. On the other hand, if $\lambda$ is a sufficiently large constant, then we can optimize the constantly many bits of the last counter and $b_0$ simultaneously in just one step. In this case, we may even use $p=e^{-o(n)}$, i.e., for all such $p$ there are $(1+\lambda)$ $p$-Monte Carlo black-box algorithms using only $O(n/\! \log \lambda)$ generations. Despite these small failure probabilities it is still not clear how to derive an upper bound on the corresponding Las Vegas complexities.

\section{The \texorpdfstring{$(\mu+1)$}{(m+1)} Elitist Black-Box Complexities of OneMax for \texorpdfstring{$\mu>2$}{m>2}}
\label{sec:mu1upper}

As mentioned earlier, the $(\mu+1)$ model is quite powerful as it allows to store information about the search history. We shall use this space to implement a variant of the \emph{random sampling} optimization strategy of Erd{\H{o}}s and R{\'e}nyi~\cite{ErdR63} (see Section~\ref{sec:onemaxbackground}).
To apply this \emph{random sampling} strategy in our setting, we need to make this approach satisfy 
the ranking-basedness condition, 
the memory-restriction, and 
the elitist selection requirement. 
Luckily, the first two problems have been solved in previous works, though not for both restrictions simultaneously (see Section~\ref{sec:onemaxbackground}).

In the elitist model we do not obtain absolute fitness values but merely learn the ranking of the search points induced by the fitness function. It has been shown in~\cite{DoerrW14ranking} that the ranking-restriction does not change the complexity of the random sampling strategy by more than an at most (small) constant factor. 
That is, there exists a function $t(n) = O(n/\! \log n)$ such that for $n$ large enough the ranking of a sequence $s_1, \ldots, s_{t(n)}$ of random strings in $\{0,1\}^n$ induced by the \onemax-function uniquely determines the target string with probability at least $1-O(\sqrt{n}\exp(-\Delta \sqrt{n}/\log n))$, where $\Delta$ is some positive constant.

By the restricted memory we may not be able to store all $t(n)$ search points. But, following previous work (see for example~\cite{DoerrW14arity} for a description of this method invented in~\cite{DoerrJKLWW11}), we can split the string into smaller blocks of size $m$ each such that $t(m)\leq \mu$.
We then optimize these $n/t(m)$ blocks iteratively. Note that this is different from the strategy in Section~\ref{sec:1lambdaupper}, where all $2^{t}$ possible entries for a block of size $t$ are sampled.
 
The last challenge that we need to handle is the elitist selection. Intuitively, if we replace after the $i$-th phase (in which we sampled the required search points for optimizing the $i$-th block) the entries in the $i$-th block by the optimal ones, this should give us enough flexibility (in terms of fitness increase) to replace the entries in the $(i+1)$-st block by the random samples $s_1, \ldots, s_t$ needed to determine the optimal entries of the $(i+1)$-st block. The theorem below shows that this is indeed possible, with high probability.

\begin{theorem}
\label{thm:mu1upper}
For constant $\mu$, the $(\mu+1)$ (Monte Carlo and Las Vegas) elitist black-box complexity of \onemax is $\Theta(n)$.

For $\mu=\omega(\log^2 n/\! \log\log n) \cap O(n/\! \log n)$
the $(\mu+1)$ Monte Carlo elitist black-box complexity of \onemax is $\Theta(n/\! \log \mu)$.

There exists a constant $C>1$ such that for $\mu\geq Cn/\! \log n$, the $(\mu+1)$ (Monte Carlo and Las Vegas) elitist black-box complexity is $\Theta(n/\! \log n)$.
\end{theorem}

\begin{proof}
The lower bounds follow from Theorem~\ref{thm:lower}. For constant $\mu$ the upper bound follows the (2+1) elitist algorithm in Theorem~\ref{thm:21upper}. 
 
The result for $\mu\geq Cn/\! \log n$ follows from the result on the ranking-based black-box complexity in~\cite{DoerrW14ranking}. For the Las Vegas result recall that, as commented in~\cite[Section 3.2]{DoerrW14arity}, the random sampling technique of Erd{\H{o}}s and R{\'e}nyi can be derandomized; that is, there exists a function $t(n)=O(n/\! \log n)$ and sequences $s_1, \ldots, s_{t(n)} \in \{0,1\}^n$ such that the fitness values of these samples uniquely determine the target string of the \onemax function. This, together with the ranking-based strategy of~\cite{DoerrW14ranking} implies the upper bound in the third statement. For the lower bound, a simple information-theoretic argument shows that if the target string is uniformly at random, then with high probability $n/(2\log n)$ samples are not enough to find the optimum~\cite{ErdR63}.

To prove the statement for intermediate values of $\mu$, note that it suffices to show the case $\mu=\omega(\log^2 n/\! \log\log n) \cap O(n/\! \log^2 n)$. The case $\mu' = \omega(n/\! \log^2n) \cap O(n/\! \log n)$ follows from the case $\mu = n/\! \log^2 n$ since the $(\mu'+1)$-complexities can only be smaller than the corresponding $(\mu+1)$-complexities, and $O(n/\! \log \mu) = O(n/\! \log \mu') = O(n/\! \log n)$.

So we may assume that $\mu=\omega(\log^2 n/\! \log\log n) \cap O(n/\! \log^2 n)$. Let $k=\Theta(\mu \log \mu)=\omega(\log^2 n)$ such that for some $t\leq \mu$ the ranking of a random sequence $s_1, \ldots, s_t \in \{0,1\}^k$ induced by the \onemax values of an arbitrary \onemax function $\OM_z$ determines the target string $z$ with probability at least $1-\delta\sqrt{k}\exp(-\Delta\sqrt{k}/\log k)$, $\delta$ and $\Delta$ being the constants implicit in the result of~\cite{DoerrW14ranking}.

\vspace{1.5ex}
\textbf{Setting up the counter: }
The algorithm starts by building a \emph{neutral counter} (a counter as in Lemma~\ref{lem:11counter}) for counting values from one to $\lceil n/k \rceil$. As in previous proofs we denote the counter by $C$. Its length is $O(\log n)$.

We initialize the algorithm by sampling the string with all zeros in the first $|C|+3$ positions and random entries in the remaining positions. 
We place $C$ in the positions $\{4,\ldots, |C|+3\}$, and the optimal entries of $C$ will be copied into part $C'$, which is placed in positions $\{|C|+4,\ldots, 2|C|+3\}$. 
First we use the $(2+1)$ linear optimization strategy from Theorem~\ref{thm:21upper} to optimize part $C$. 
This requires $O(|C|)=O(\log n)$ (deterministic) iterations. 
At the end of this phase we set the first bit to one, indicating that we are  now ready to copy $C$ into $C'$. We do so by applying the strategy from Lemma~\ref{lem:copy} with $B:=\{2|C|+4, \ldots, n\}$.

When $C$ is copied into $C'$ we flip the second flag bit and continue by initializing the counter. This requires to flip $|C|/2$ bits from the correct into their non-optimal state. Again we apply Lemma~\ref{lem:copy} with $B$ as above. Note that by Chernoff's bound, $B$ satisfies the requirements of Lemma~\ref{lem:copy} with high probability.

By comparison of $C$ with $C'$ we recognize when the counter is initialized. We are then ready to enter the main part of the algorithm in which we optimize part $B$. 
Note that at this point the first three bits are $110$. 
Note further that at most $O(\log n)$ bits in $B$ have been touched at this point, so that, as also commented in the proof of Theorem~\ref{thm:11upper} in Section~\ref{sec:11upperProof}, by Chernoff's bound,
with probability at least $1-\exp(-\eps^2n/3)$, after this copy operation at least a $1/2-\eps$ fraction of $B$ is non-optimized, for any constant $\eps>0$. The $-\eps^2n/3$ part (en lieu of the typical $-\eps^2n/2$ expression) in this bound accounts for the fact that $O(\log n)$ bits have been optimized during the implementation and initialization of the counter and the fact that we regard the substring $B$ of size only $n-O(\log n)$.

\vspace{1.5ex}
\textbf{Optimization of the Main Part Using Random Sampling: }
We divide part $B$ into blocks of length $k$ each; only the last block, which will be treated differently, may have smaller size.
We aim at optimizing the blocks iteratively. 

To this end, we first show that with high probability we can for each of the $\lceil |B|/k \rceil$ blocks determine the target entries in the block from the $t$ random samples. Recall that for each block individually this probability is $1-\delta\sqrt{k}\exp(-\Delta\sqrt{k}/\log k)$. By a union bound, the probability that it works for all $\lceil |B|/k \rceil$ blocks is thus at least 
$1-\delta (n/\sqrt{k}) \exp(-\omega(\sqrt{\log^2 n})) = 1-O(n^{-c})$ for any positive constant $c$. 

Fix $0< \eps,\eps'< 1/6$. We show next that with high probability the fitness contribution of each block (except for, potentially, the last one, which we can and do ignore in the following) is between $(1/2-\eps)k$ and $(1/2+\eps)k$ initially.
After initialization of the algorithm, the expected fitness contribution of each block is $1/2$ times the length of the block, i.e., $k/2$. 
During the setup and the initialization of the counter, we have changed at most $O(\log n)$ bits in $B$, their positions being uniformly distributed in $B$. 
Therefore, each block has an expected fitness contribution after the setup and initialization of the counter of $(1-o(1))k/2$. By Chernoff's bound, its contribution is between the desired $(1/2-\eps)k$ and $(1/2+\eps)k$ with probability at least $1-2\exp(-\delta k)$ for some positive constant $\delta$. 
By a union bound, the fitness contribution of every (but potentially the last) block is thus between $(1/2-\eps)k$ and $(1/2+\eps)k$ with probability at least $1-2(n/k)\exp(-\delta k)$.
Together with the requirement $k=\omega(\log^2 n)$ this shows that the failure probability is at most $n^{-c}$ for any positive constant $c$. 
We may therefore condition all the following statements on this event. 

By the same reasoning as above, the probability that for all blocks $i$ and for all $j \in [t]$ the fitness contribution of the random string $s_j$ in block $i$ is between $(1/2-\eps')k$ and $(1/2+\eps')k$ is at least $1-(n/\! \log k)\exp(-\delta' k)$, for some positive constant $\delta'$. As above, this expression is at least $1-n^{-c}$ for any positive constant $c$. We may therefore also condition on this event. 

Let us assume that for some block $1 \leq i<\lceil |B|/k \rceil-1$ we have sampled the required $t$ random strings. We show how to optimize block $i+1$. (The optimization of the first and the last block needs to be handled differently and will be considered below.) 
That is, the entries of the first $i-1$ blocks are already optimized, the counter of the $\mu$ strings in the population is set to $i$ and the entries in the $i$th block of $t$ of these strings are taken from $\{0,1\}^k$ uniformly at random.\footnote{In more precision, one substring is the \emph{median query} required for the ranking-based algorithm from~\cite{DoerrW14ranking}. See Lemma 12 in~\cite{DoerrW14ranking} for the details of this query, which is needed to verify that the fitness level $k/2$ is correctly identified. It is only important for us to know that we need to make one additional non-random query, the fitness contribution of which is $\lceil k/2 \rceil$ with very high probability. We ignore this query in this presentation, as it is obvious that it does not create any problems with our approach.}
The next $t$ queries are as follows. In each query, we replace the entries in the $i$-th block by the optimal ones, we increase the neutral counter by one, and we replace the entries in the $(i+1)$-st block by entries that are taken from $\{0,1\}^k$ independently and uniformly at random. Let us first argue that these queries will be accepted into the population. 
When we replace the initial entries in block $i+1$ by the random string $s_j$ we lose a fitness contribution of at most $(\eps+\eps')k$. 
On the other hand, we have a fitness increase of at least $(1/2-\eps')k>(\eps+\eps')k$ from replacing the random entries in the $i$th block by the optimal ones. The neutral counter does not have any effect on the fitness and can thus be ignored. 

It remains to describe how to optimize the first and the last block.
For the last block, we simply use the (2+1) linear elitist optimization strategy of Theorem~\ref{thm:21upper}. Since the size of this block is at most $k=\Theta(\mu \log \mu) = O(n/\! \log \mu)$, this does not affect the overall runtime by more than a constant factor. Of course, we increase in each query for the last block the neutral counter by one and we replace the random strings in the penultimate block by the optimal ones. 

Getting the desired random samples into the first block is a bit more challenging. We need to respect the elitist selection rule and need thus to make sure that the random samples in the first block are accepted. A simple trick enables us to guarantee that. We first optimize the first block with the linear (2+1) elitist strategy from Theorem~\ref{thm:21upper} (note again that the size of the block is $k=O(n/\! \log \mu)$). We then invert all the bits in the first block by applying the strategy from Lemma~\ref{lem:copy} to the first block and part $B$. Since this affects at most $k=O(n/\! \log \mu)=o(n)$ bits in $B$, all the probabilistic statements that we made about $B$ above still hold. At this point the fitness contribution of the first block is zero and all random samples in the first block can thus be accepted. 

There are $\lceil |B|/k \rceil=O(n/k)$ blocks in total. For each block we sample $t=O(k/\log k)$ random strings to determine the optimum. The overall number of samples performed in this phase of the algorithm is thus 
$O(n/\! \log k) = O(n/\! \log \mu)$. 

\vspace{1.5ex}
\textbf{Optimizing the Counter: }
Once all the blocks have been optimized, that is, as soon as $s_{t}$ has been sampled in the last block, we sample the search point which replaces $s_{t}$ by the optimal entries for this block and which has the third bit set to one (it has been zero in all previous iterations). 
This indicates that we can now go to the next phase, in which we optimize all the bits in positions $\{1,2\} \cup \{4, \ldots, 2|C|+3\}$ using the linear (2+1) elitist strategy from Theorem~\ref{thm:21upper}. Finally, we check if replacing the third bit by a zero improves the fitness further. 
This last phase is deterministic and requires $O(|C|)=O(\log n)$ queries. 
\end{proof}

\section{Remark on \texorpdfstring{$(\mu,\lambda)$}{(m,l)} Elitist Black-Box Complexities}
\label{sec:comma}

It is interesting to note that it can be significantly easier in the elitist black-box model to optimize a function when allowed to use so-called \emph{comma} strategies instead of the \emph{plus} strategies described by Algorithm~\ref{alg:elitist}. 
To make things formal, we call an algorithm that follows the scheme of Algorithm~\ref{alg:elitist} with Line~8 replaced by 
\[\text{Set $X \assign \{y^{(1)},\ldots,y^{(\lambda)}\}$}
\] 
and Line~9 running only to $\lambda-\mu$ a $(\mu,\lambda)$ elitist algorithm.
That is, a $(\mu,\lambda)$ elitist algorithm has to keep in each iteration the $\mu$ best sampled offspring, but it is allowed (and forced) to ignore the parent solutions (which, consequently, can be of better fitness). As mentioned in the introduction, the term \emph{elitist selection} may not be appropriate here, depending on the context, and \emph{truncation selection} may be the preferable expression. In any case, if the algorithm wants to maintain parts of the parental population, it can simply resample those individuals that should be kept.

Note in particular that $(\mu,\lambda)$ elitist algorithms can do restarts. Therefore, as discussed in Section~\ref{sec:lasvegas}, to bound the Las Vegas complexity of $(\mu,\lambda)$ elitist algorithms, it suffices to bound its Monte Carlo complexity. Note further that for all $\lambda'$ with $\mu+\lambda' \leq \lambda$ we can imitate every $(\mu+\lambda')$ elitist black-box algorithm by a $(\mu,\lambda)$ elitist black-box algorithm, from Theorems~\ref{thm:11upper} and~\ref{thm:1lambdaupper} we get the following corollary. 

\begin{corollary}
\label{cor:comma}
The (1,2) (Las Vegas and Monte Carlo) elitist black-box complexity of \onemax is at most $O(n)$. 

For any $\lambda\geq 2$, there are $(1,\lambda)$ (Las Vegas and Monte Carlo) elitist black-box algorithms that need at most $O(n/\! \log \lambda)$ generations on \onemax.
\end{corollary}

Asymptotically, these bounds are tight, since matching lower bounds can be obtained by the same information-theoretic arguments as used in Theorem~\ref{thm:lower}. We can easily improve the bounds in Corollary~\ref{cor:comma} as follows.
\begin{theorem}
\label{thm:11comma}
The (1,2) Las Vegas elitist black-box complexity of \onemax is at most $2n+1$, and the correspnding algorithm needs at most $n+1$ generations. 

For any $\lambda \geq 2$ there are $(1,\lambda)$ Las Vegas and Monte Carlo elitist black-box algorithms that need at most $\lceil n/\! \lfloor\log_2 \lambda \rfloor \rceil$ generations on \onemax.
\end{theorem}

\begin{proof}
We first regard the (1,2) situation.
Initialize the algorithm with the string $x_1=(1,0,\ldots,0)$. 
We maintain the following invariant: at the beginning of iteration $t$ the string $x_t$ in the memory has entry 1 in position $t$ and zeros in all positions $i \geq t$.

We sample in iteration $t$ the two search points $x_t \oplus e_{t+1}$ and $x_t \oplus e_t \oplus e_{t+1}$, where for all $j \in [n]$ the string $e_j$ is the string with all entries except the $j$-th one set to zero. That is, we either flip only the $t+1$-st bit in $x_t$ or we flip both the $t$-th and the $t+1$-st bit. Since the two offspring differ in exactly one position, one of them has strictly better fitness than the other, and we (necessarily) keep the better one. At the end of the $t$-th iteration the search point in the memory is thus optimized in the first $t$ positions. After the $n$-th iteration, the optimum is found. 

For the $(1,\lambda)$ situation we simply apply the previous idea with an exhaustive search on blocks of length $\ell:=\lfloor\log_2 \lambda \rfloor$. That is, we always move the one by $\ell$ positions to the right while at the same time testing all possible $2^{\ell} \leq \lambda$ possible entries in these $\ell$ positions.

Both algorithms are deterministic and therefore Las Vegas.
\end{proof}

For completeness, we note that the (1,1) (Las Vegas or Monte Carlo) complexity of \onemax is $\Theta(2^n)$. The Las Vegas upper bound is given by random sampling, and it implies the Monte Carlo upper bound as discussed in Section~\ref{sec:lasvegas}. For the lower bound, note that the algorithm does not get any information about the search point it stores, except whether it is the optimum. Therefore, the problem is at least as hard as the needle-in-haystack problem \textsc{Needle} where all search point except the optimum have the same fitness. Even if we give the algorithm access to infinite memory, for any $0<c<1$ after $c2^{n}$ steps the optimum of \textsc{Needle} will not be found with probability at least $1-c$, proving the lower bounds.

\section{Conclusions}
\label{sec:conclusions}
We have analyzed black-box complexities of \onemax with respect to $(\mu+\lambda)$ memory-restricted ranking-based algorithms. Moreover, we have shown that the complexities do not change if we also require the algorithms to be elitist, provided that we regard Monte-Carlo complexities. For different settings of $\mu$ and $\lambda$ we have seen that such algorithms can be fairly efficient and attaining the information-theoretic lower bounds. 

An interesting open question arising from our work is a tight bound for the Las Vegas complexity of \onemax in the (1+1) elitist black-box model. We have sketched in Section~\ref{sec:11upper} the main difficulties in turning our Monte Carlo algorithm into a Las Vegas heuristic. The possible discrepancy between these two notions also raises the question which problems can be optimized substantially more efficiently with restarts than without, an aspect for which some initial findings can be found in the literature, e.g.,~\cite{Jansen02}, but for which no strong characterization exists. 

\subsection*{Acknowledgments}
This research benefited from the support 
of the ``FMJH Program Gaspard Monge in optimization and operation research'', 
and from the support to this program from EDF.}


\newcommand{\etalchar}[1]{$^{#1}$}
\providecommand{\bysame}{\leavevmode\hbox to3em{\hrulefill}\thinspace}
\providecommand{\MR}{\relax\ifhmode\unskip\space\fi MR }
\providecommand{\MRhref}[2]{%
  \href{http://www.ams.org/mathscinet-getitem?mr=#1}{#2}
}
\providecommand{\href}[2]{#2}

\end{document}